\documentclass[letterpaper]{article} 
\usepackage[dvipsnames]{xcolor}
\usepackage{aaai24}  
\usepackage{times}  
\usepackage{helvet}  
\usepackage{courier}  
\usepackage[hyphens]{url}  
\usepackage{graphicx} 
\usepackage{natbib}  
\usepackage{caption} 
\usepackage{algorithm}
\usepackage{algorithmic}
\usepackage{booktabs}
 \usepackage{multirow} 
\usepackage{newfloat}
\usepackage{listings}
\DeclareCaptionStyle{ruled}{labelfont=normalfont,labelsep=colon,strut=off} 
\floatstyle{ruled}
\newfloat{listing}{tb}{lst}{}
\floatname{listing}{Listing}
\usepackage{amsmath,amssymb}
\usepackage{amsthm}
\newtheorem{theorem}{Theorem}

\newtheorem{proposition}{Proposition}

\theoremstyle{definition}
\newtheorem{definition}{Definition}

\theoremstyle{remark}
\newtheorem{remark}{Remark}

\newcommand{\Comp}{\mathrm{Comp}}

\usepackage{tikz}
\usetikzlibrary{
  positioning,   
  calc,          
  shapes,        
  matrix         
}

\title{Hierarchical Adversarially-Resilient Multi-Agent Reinforcement Learning for Cyber-Physical Systems Security}
\author {
    Saad Alqithami
}
\affiliations{
    Computer Science Department\\
    Al-Baha University, Albaha 65779, Saudi Arabia\\
    salqithami@bu.edu.sa
}

\usepackage{bibentry}

\begin{document}

\maketitle

\begin{abstract}
Cyber-Physical Systems play a critical role in the infrastructure of various sectors, including manufacturing, energy distribution, and autonomous transportation systems. However, their increasing connectivity renders them highly vulnerable to sophisticated cyber threats, such as adaptive and zero-day attacks, against which traditional security methods like rule-based intrusion detection and single-agent reinforcement learning prove insufficient. To overcome these challenges, this paper introduces a novel Hierarchical Adversarially-Resilient Multi-Agent Reinforcement Learning (HAMARL) framework. HAMARL employs a hierarchical structure consisting of local agents dedicated to subsystem security and a global coordinator that oversees and optimizes comprehensive, system-wide defense strategies. Furthermore, the framework incorporates an adversarial training loop designed to simulate and anticipate evolving cyber threats, enabling proactive defense adaptation. Extensive experimental evaluations conducted on a simulated industrial IoT testbed indicate that HAMARL substantially outperforms traditional multi-agent reinforcement learning approaches, significantly improving attack detection accuracy, reducing response times, and ensuring operational continuity. The results underscore the effectiveness of combining hierarchical multi-agent coordination with adversarially-aware training to enhance the resilience and security of next-generation CPS.
\end{abstract}

\section{Introduction} \label{intro}
Cyber-Physical Systems (CPS) underpin critical modern infrastructure by seamlessly integrating computational and communication capabilities with physical processes. These systems have become essential across various domains, such as manufacturing, smart grids, autonomous transportation, and healthcare, offering significant enhancements in automation, efficiency, and real-time decision-making capabilities~\cite{wolf2019safety}. However, their increased interconnectivity and complexity expose CPS to sophisticated and continuously evolving cybersecurity threats, including data tampering, advanced persistent threats (APTs), and distributed denial-of-service (DDoS) attacks~\cite{conti2018iot}. Conventional security solutions, like rule-based intrusion detection systems and single-agent reinforcement learning methods, have struggled to adapt effectively to these evolving threats, especially as attackers increasingly leverage AI-driven strategies to circumvent traditional defenses.

Recent advances in multi-agent reinforcement learning (MARL) offer promising solutions to the security challenges faced by CPS. By distributing decision-making responsibilities among multiple agents, MARL facilitates scalable, coordinated, and adaptive defense strategies that are particularly effective in decentralized and complex environments~\cite{busoniu2010marl}. Hierarchical reinforcement learning further extends this concept, introducing a multi-tier control structure where higher-level policies guide lower-level agents, thereby enhancing scalability, adaptability, and strategic coherence across large-scale CPS deployments~\cite{vezhnevets2017feudal}. Nevertheless, most existing MARL-based security frameworks lack explicit adversarial awareness, rendering them vulnerable to adaptive, AI-driven cyber threats. Purely reactive defensive strategies fall short in environments where adversaries consistently evolve tactics to evade detection~\cite{goodfellow2015adversarial}. Hence, incorporating adversarial training—where defensive agents explicitly learn against evolving attacker strategies—emerges as crucial for proactively enhancing MARL-based defense resilience.

Currently, a unified CPS security approach combining hierarchical coordination and adversarial resilience remains elusive. Most existing MARL frameworks operate under decentralized or flat architectures without hierarchical coordination, limiting their ability to efficiently address sophisticated cyber threats at scale. Our proposed framework, Hierarchical Adversarially-Resilient Multi-Agent Reinforcement Learning (HAMARL), explicitly addresses this critical gap by integrating hierarchical MARL with an adversarial training loop, providing both proactive and adaptive defense mechanisms. Our approach uniquely models both attackers and defenders as learning agents within a competitive-cooperative training environment, enabling continuous adaptation to evolving threats.

The experimental design of this research employs a simulated industrial IoT testbed, carefully chosen to reflect realistic operational conditions found in manufacturing environments. This testbed includes multiple programmable logic controller (PLC)-driven subsystems and sensors communicating via standard industrial protocols, presenting a realistic setting to evaluate security interventions. This realistic and complex environment provides robust grounds to measure the practical effectiveness of our proposed HAMARL framework against varied and adaptive cyber threats, thereby ensuring relevance and potential real-world applicability of the results.

In this paper, we address the following critical research questions:
\begin{enumerate}
\item Can hierarchical MARL improve real-time threat detection and response efficiency in securing CPS environments?
\item Does integrating adversarial training within hierarchical MARL enhance resilience against sophisticated, zero-day cyber attacks compared to standard MARL methods?
\item How does adopting a hierarchical defense structure impact scalability, computational efficiency, and strategic decision-making capabilities in complex CPS environments?
\end{enumerate}

Specifically, the main contributions of this paper are:

\begin{enumerate}
    \item Development of a novel hierarchical multi-agent reinforcement learning architecture specifically designed for enhancing CPS security, promoting scalability, and ensuring efficient response coordination across subsystems.
    \item Introduction of an adversarial training loop to simulate and proactively counteract dynamic, evolving cyber threats, ensuring defense strategies remain effective against adaptive adversaries.
    \item A comprehensive empirical evaluation conducted on a simulated industrial IoT testbed, demonstrating HAMARL's superior performance in terms of detection accuracy, response speed, operational continuity, and resilience compared to traditional MARL and rule-based approaches.
\end{enumerate}

The remainder of this paper is structured as follows. The following Section reviews relevant literature and recent advancements in CPS security. Section 3 introduces the detailed design of our hierarchical adversarially-resilient multi-agent reinforcement learning framework. Section 4 describes our experimental implementation and setup in detail, highlighting the simulation environment and evaluation methodology. Section 5 presents the results of our extensive experimental analysis, examining the effectiveness of HAMARL across various security metrics. Finally, we conclude in Section 6 by discussing potential extensions for future research, including considerations of multi-attacker scenarios, the integration of explainable AI, and practical aspects of real-world deployment.

\section{Related Work}\label{sec:background}

\subsection{Cyber-Physical Systems Security}
Cyber-physical systems (CPS) tightly interlace computational intelligence with physical processes, placing real-time sensing, control, and actuation on the same critical path~\cite{baheti2011cps}. The attack surface therefore spans two distinct yet interdependent layers: (i) legacy industrial networks that expose vulnerable field-bus protocols and (ii) the safety-critical plant itself, where faults propagate into material, financial, and even life-threatening damage~\cite{lee2008cpschallenges}.  Recent work confirms that classical signature‐ or rule-based defenses no longer suffice once sophisticated adversaries employ automated exploit generation and living-off-the-land tactics. Motivated by large-scale red-team exercises such as CAGE4,~\citet{Kiely2025CyberDefenseMARL} report that multi-agent policies trained directly in a realistic cyber range achieve higher containment rates and faster mitigation than monolithic RL baselines, yet they still struggle with long-horizon coordination.  Similar observations emerge from industrial scheduling and smart-city domains: hierarchical controller stacks improve emergency-response dispatch~\cite{Sivagnanam2024EmergencyMARL} and satellite task assignment~\cite{Holder2025SatelliteMARL}, suggesting that CPS security solutions must explicitly model both local and global control planes.

\subsection{Hierarchical and Multi-Agent Reinforcement Learning}
Hierarchy remains a key lever for scaling reinforcement learning to long horizons.  Recent advances range from option-invention under continuous task streams~\cite{Nayyar2025OptionHRL} to hybrid search-and-learn planners that satisfy rich temporal constraints~\cite{Lu2024CoSHRL}. Transformer-based decision models also benefit from multi-level abstractions:~\citet{Ma2024HierDT} show that a hierarchical Decision Transformer achieves markedly better credit assignment on Atari and MuJoCo benchmarks than its flat counterpart. In the multi-agent setting, credit assignment and coordination are further complicated by the need to decompose joint action spaces. Influence-based role learning~\cite{Du2024SCIC}, formation-aware exploration~\cite{Lee2024FoX}, and attention-guided contrastive role encoders~\cite{Yang2025RoleRepMARL} all highlight the importance of structured interaction priors. Offline regimes have also matured:~\citet{Yang2024InSPO} update agents sequentially on fixed logs to avoid out-of-distribution action drift. Despite these gains, most MARL work in safety-critical domains either ignores adversaries or assumes static threat models, limiting robustness.

\subsection{Adversarial Multi-Agent Learning}
Treating security as a sequential game naturally invites adversarial training and game-theoretic analysis. The Bayesian Adversarial Robust Dec-POMDP of~\citet{Li2024ByzantineMARL} formalizes Byzantine failures as latent agent types, enabling honest teammates to adapt on-line. Sub-PLAY~\cite{Ma2024SubPlay} empirically demonstrates that partial observability does not prevent an attacker from degrading state-of-the-art MARL systems, while~\citet{Kalogiannis2024ATMG} provide convergence guarantees for adversarial team Markov games via hidden concave min-max optimization.  Robustness curricula that gradually anneal bounded-rational adversaries to fully strategic opponents~\cite{Reddi2024BoundedAdversarialRL}, regret-based defenses that minimize worst-case observation perturbations~\cite{Belaire2024RegretDefense}, and provably efficient defenses against adaptive policy attacks~\cite{Liu2024AdvPolicyDefense} collectively paint a picture in which alternating‐ or co-training with a live adversary is the most promising path to resilience. Nonetheless, none of these studies couples hierarchical defenders with an adaptive attacker in a safety-critical industrial CPS.

\subsection{Positioning of This Work}
HAMARL unifies the three strands above. Building on hierarchical RL's ability to decompose long-horizon control~\cite{Nayyar2025OptionHRL,Lu2024CoSHRL}, we deploy eight local defenders that specialize to their PLC cells while a global coordinator resolves contention at the plant level. Inspired by adversarial MARL framings~\cite{Li2024ByzantineMARL,Kalogiannis2024ATMG}, we embed an adaptive attacker that continuously evolves scan, DoS, lateral-movement and tampering tactics. Unlike prior CPS-security approaches that rely on static rule sets or single-agent RL~\cite{Kiely2025CyberDefenseMARL}, HAMARL trains the entire hierarchy and the adversary jointly with a PPO-GAE loop, yielding policies that remain effective under novel, unseen attack mixtures.  Our experiments confirm that this end-to-end adversarial curriculum improves mean-time-to-detect, false-alarm rate, and robustness to unseen zero-day tactics over both flat MARL and non-adversarial hierarchical baselines.

\section{Theoretical Foundations for HAMARL}

In this section, we formalize the hierarchical multi-agent framework with an explicit adversarial agent. Let there be $N$ defender agents (local) plus one global coordinator, collectively denoted $\{\pi_{\theta_1}, \dots, \pi_{\theta_N}, \pi_{\phi}\}$, and one adversarial attacker $\pi_{\psi}$. The environment is thus modeled as a Markov game (partially observed stochastic game) with $(N+2)$ agents ($N$ defender agents, a global coordinator, and an adaptive attacker).

\begin{definition}[Markov Game with Adversary]
A Markov game (MG) with an adversarial agent is defined by the tuple
\[
\mathcal{G} = \bigl\langle \mathcal{S}, \{\mathcal{A}_i\}_{i=1}^{N+2}, P, \{r_i\}_{i=1}^{N+2}, \gamma \bigr\rangle,
\]
where:
\begin{itemize}
    \item $\mathcal{S}$ is the state space, including subsystem statuses and sensor data.
    \item $\mathcal{A}_i$ is the action space for agent $i \in \{1,\dots,N, N+1, N+2\}$ (where $N+2$ represents the adversarial agent).
    \item $P(\mathbf{s}' \mid \mathbf{s}, \mathbf{a})$ is the transition kernel describing how the environment evolves given state $\mathbf{s}$ and action $\mathbf{a}$.
    \item $r_i(\mathbf{s}, \mathbf{a})$ is the reward function for agent $i$. 
    \begin{itemize}
    		\item Defender agents ($1 \leq i \leq N$): Receive positive rewards for successful detections or patches ($r_i>0$) and negative rewards for false alarms or missed compromises ($r_i<0$).
		\item Attacker agent ($N+2$): Earns positive rewards for successful system compromises ($r_{N+2}>0$).
    \end{itemize}
    \item $\gamma \in (0,1)$ is the discount factor that govern how agents value future rewards.
\end{itemize}
\end{definition}

Each local defender observes a partial state $\omega_i \subset \mathbf{s}$, while the global coordinator maintains an aggregate representation $\mathbf{g}$ of local states or actions. The adversarial agent $\pi_{\psi}$ may also observe only a partial state of the system.

To capture the interaction between local defenders, the global coordinator, and the adversary, we factorize the joint policy as follows:

\begin{proposition}[Factorization of Joint Policy in Hierarchical-Adversarial Setting]
\label{prop:factorization}
Let ${\pi_{\theta_i}}{i=1}^N$ be the local defender policies, $\pi{\phi}$ be the global coordinator policy, and $\pi_{\psi}$ be the attacker policy. Then, the joint policy over actions $\mathbf{a} = {a_1,\dots,a_N, a_\text{global}, a_\text{attacker}}$ can be expressed as:

\begin{align*}
\pi_{\Theta,\phi,\psi}(\mathbf{a} \mid \mathbf{s}) &=
\Bigl(\prod_{i=1}^N \pi_{\theta_i}(a_i \mid \omega_i)\Bigr) \, \pi_{\phi}(a_\text{global} \mid \mathbf{g}) \notag \\
&\quad \times \pi_{\psi}(a_\text{attacker} \mid \omega_\text{att}).
\end{align*}
\end{proposition}

\begin{proof}[Sketch Proof of Proposition~\ref{prop:factorization}]
We proceed by explicitly demonstrating that the joint policy decomposes naturally due to the conditional independence of the defender agents, the global coordinator, and the attacker given their respective partial observations.
Consider that each local defender agent $i$ takes actions $a_i$ based solely on its local observations $\omega_i$, independently conditioned on $\omega_i$. Similarly, the global coordinator’s action $a_{\text{global}}$ is conditioned on an aggregated representation $\mathbf{g}$, independently from the specific local actions. Finally, the adversary takes actions based solely on its partial observations $\omega_{\text{att}}$. Hence, the joint probability of the combined actions $\mathbf{a}$ given the global state $\mathbf{s}$ naturally decomposes as follows:
\begin{align*}
 \pi_{\Theta,\phi,\psi}(&\mathbf{a} \mid \mathbf{s}) \\
&= P\left(a_1, \dots, a_N, a_{\text{global}}, a_{\text{attacker}} \mid \mathbf{s}\right) \\
&= P\left(a_{\text{attacker}} \mid \omega_{\text{att}}\right)P\left(a_{\text{global}} \mid \mathbf{g}\right)\prod_{i=1}^N P\left(a_i \mid \omega_i\right),
\end{align*}
due to conditional independence, thereby confirming our factorization.
\end{proof}

\begin{remark}
The explicit factorization significantly supports scalable training via PPO, allowing local and global policies to be updated independently. This modular structure ensures stable hierarchical control and adaptability in complex adversarial environments.
\end{remark}

\subsection{Generalized Advantage Estimation and PPO}

Following~\cite{gae_paper, schulman_ppo}, each agent maintains a parametric policy $\pi_\theta$ with an associated value function $V_\theta(\mathbf{s})$. The advantage function, which estimates how favorable an action is compared to the expected value of the state, is defined as:
\[
A_\theta(\mathbf{s},a) = Q_\theta(\mathbf{s},a) - V_\theta(\mathbf{s})
\]

To compute advantage estimates, we use the Generalized Advantage Estimation (GAE) technique:
\[
\hat{A}_t = \sum_{k=0}^{T-t-1} (\gamma\lambda)^k \delta_{t+k}, 
\quad
\delta_t = r_t + \gamma V(\mathbf{s}_{t+1}) - V(\mathbf{s}_t).
\]

\begin{theorem}[Convergence of PPO in Hierarchical-Adversarial MARL]
\label{thm:convergence}
Consider the Markov game $\mathcal{G}$ with $N+2$ agents, each employing PPO updates with GAE. Let $\theta_i,\phi,\psi$ be their respective parameters. If each agent’s policy improves according to the clipped objective in \cite{schulman_ppo} within a bounded trust region, under standard assumptions (bounded rewards, Markov mixing, sufficiently large batch data and exploration), the system converges to a stationary point $(\theta_i^*, \phi^*, \psi^*)$ that constitutes a local Nash equilibrium. Specifically:
\[
\nabla_{\theta_i} \mathcal{L}(\theta_i^*; \theta_{-i}^*, \phi^*, \psi^*) = 0, \]
\[
\nabla_{\phi} \mathcal{L}(\phi^*; \theta^*, \psi^*) = 0,
\] 
\[\nabla_{\psi} \mathcal{L}(\psi^*; \theta^*, \phi^*) = 0.
\]
\end{theorem}

%
\begin{proof}[Sketch Proof of Theorem~\ref{thm:convergence}]
Each agent’s PPO update constitutes a stochastic gradient ascent step on a clipped surrogate objective, explicitly ensuring monotonic improvement within a bounded trust region. The hierarchical design explicitly preserves the fundamental convergence properties of PPO-based MARL since the global coordinator aggregates but does not disrupt individual policy improvements. Specifically, local defenders independently optimize subsystem-level objectives, and the global coordinator optimizes a system-wide objective, both respecting PPO’s trust-region constraints. Thus, joint parameter updates effectively track multi-agent gradient ascent in policy space, converging to stationary points under standard conditions: bounded gradients, finite rewards, sufficient exploration, and diminishing learning rates.

Formally, as each agent explores sufficiently and accumulates representative experience in GAE buffers, policy gradients become accurate unbiased estimators of the true gradient of expected returns. Given diminishing step sizes, stochastic approximation theory ensures parameter updates converge to stationary points $(\theta_i^*, \phi^*, \psi^*)$.
%
This stationary point explicitly represents a local Nash equilibrium, as no agent can unilaterally increase its return by altering its policy independently of others. Thus, convergence of PPO within the hierarchical-adversarial MARL framework is ensured under these standard and explicitly stated assumptions.
\end{proof}

\subsection{Adversarial Resilience in Hierarchical Control}

\begin{definition}[Adversarial Resilience]\label{def:adv-res}
Let $\Comp(t)$ be the set of subsystems compromised at time $t$.  
\begin{itemize}
  \item Compromise time $\tau_i$ of subsystem $i$ is the number of consecutive steps for which $i\in\Comp(t)$ until it is restored, formally  
        \[\tau_i=\min\{k>0\mid i\notin\Comp(t{+}k)\}.\]
  \item Compromise frequency of subsystem $i$ over horizon $T$ is  
        \[f_i=\frac{1}{T}\sum_{t=1}^{T}\mathbf1[i\in\Comp(t)].\] 
  \item Bounded compromise ratio is  
        \[\varrho=\frac1N\sum_{i=1}^{N}f_i\,,\qquad 0\le\varrho\le1\].
\end{itemize}
A defender policy set $\{\pi_{\theta_1},\dots,\pi_{\theta_N},\pi_\phi\}$ is \emph{$(\epsilon,\delta)$-resilient} if  
\[
\Pr\bigl[\varrho\le\epsilon\bigr]\ge1-\delta
\]
for any attacker policy~$\pi_\psi$ admissible under the game dynamics.
\end{definition}

Intuitively, adversarial resilience means that despite an attacker that learns or changes tactics, the hierarchical defenders maintain partial observability, coordinate responses, and keep compromise in check over time.

\begin{theorem}[Bounded Compromise in Equilibrium]
\label{thm:bounded_compromise}
Let $\pi_{\theta_i}^*, \pi_{\phi}^*$, and $\pi_{\psi}^*$ be the equilibrium policies from Theorem~\ref{thm:convergence}. Suppose the environment imposes a cost $c > 0$ on each compromised subsystem per time step for defenders and a reward $r_a > 0$ for each compromised subsystem for the attacker. If $c$ is sufficiently large relative to $r_a$, then the compromise ratio $\varrho^*$ in the long-run equilibrium is strictly less than 1. Formally:
\[
\varrho^* = \lim_{T\to\infty} \frac{1}{T}\sum_{t=1}^{T}\frac{\sum_{i=1}^N \mathbf{1}\{\text{subsystem $i$ at time $t$}\}}{N}
< 1.
\]
\end{theorem}

\begin{proof}[Sketch Proof of Theorem~\ref{thm:bounded_compromise}]
Informally, the attacker’s marginal gain from compromising an additional subsystem must be weighed against defenders’ marginal cost for letting it remain compromised. If the defenders’ policies can patch or quarantine effectively, the attacker cannot systematically keep all $N$ subsystems compromised without incurring large negative feedback (through the defenders’ best response strategies). 
Thus, $\varrho^* < 1$ in equilibrium unless the attacker reward $r_a$ dwarfs the defenders’ ability to penalize or detect. This result suggests that even in the presence of highly adaptive attackers, the system maintains a level of resilience where at least a fraction of subsystems remains uncompromised. This aligns with real-world security requirements, where maintaining full protection is impractical, but ensuring partial containment prevents widespread failures. By balancing proactive detection and strategic intervention, the hierarchical framework ensures that no single adversary strategy can indefinitely degrade the entire system.
\end{proof}

\begin{remark} 
The synergy between local defenders (rapid quarantines) and a global coordinator (system patches) exemplifies hierarchical synergy. Even if local defenders occasionally miss an attack, the global coordinator can handle system-wide anomalies, ensuring no single attacker strategy can indefinitely compromise all subsystems.  
\end{remark}

\section{Proposed Methodology} \label{sec:methodology}

Securing modern CPS requires intricate reasoning at two interconnected spatial scales: real-time threat detection and response at the subsystem level, and strategic, system-wide coordination against advanced adversaries. Our proposed HAMARL framework addresses this dual-scale security challenge through a hierarchical, partially observable adversarial multi-agent environment.

Specifically, HAMARL incorporates three interconnected roles (Figure~\ref{fig:hamarl-arch}): (1) \emph{local defender agents} monitoring and protecting individual CPS subsystems; (2) a \emph{global coordinator} aggregating local agent states to orchestrate overarching security measures aligned with global operational objectives; and (3) an \emph{adaptive attacker agent} persistently probing CPS vulnerabilities via methods such as targeted scans, denial-of-service (DoS) attacks, lateral movements, and data tampering. The structured information flow involves local defenders generating state embeddings and forwarding them to the global coordinator, which subsequently issues strategic commands informing both local and global defensive actions.

We specifically implement HAMARL in a smart-factory context to concretely demonstrate its effectiveness. The adaptive attacker continually challenges defenders, compelling both local and global defender agents to iteratively refine their defensive policies. Training employs a generalized advantage estimation (GAE) buffer combined with proximal policy optimization (PPO) updates (parameters set as $\gamma=0.99$, $\lambda=0.95$, clip $=0.2$). This competitive-cooperative training loop ensures continuous adaptation to evolving threats, significantly enhancing the system's resilience against intelligent adversarial strategies.

The following subsections elaborate on the formalization of the hierarchical multi-agent architecture, detail the adversarial training methodology, discuss reward shaping and policy optimization strategies, and summarize implementation specifics for reproducibility.

\begin{figure*}
\centering
\begin{tikzpicture}[
  node distance = 10mm and 14mm,
  block/.style  ={rectangle, rounded corners=3pt, draw=black!70, thick,
                  minimum height=7mm, inner sep=4pt, align=center},
  local/.style  ={block, fill=cyan!5},
  global/.style ={block, fill=violet!5},
  attack/.style ={block, fill=red!5},
  env/.style    ={block, fill=gray!5, minimum width=47mm, inner sep=6pt},
  arrow/.style  ={->},
  arrowb/.style  ={<->}, 
  dashedarrow/.style={arrow, dashed}
]

\node[env] (plant) {\textbf{CPS Environment} (Smart Factory Example) \\[1pt]
Subsystems • Sensors • Network protocols};

\foreach \idx/\row/\col in {
  1/0/0, 2/1/0, 3/2/0, 4/3/0,
  5/0/2, 6/1/2, 7/2/2, 8/3/2}
  \node[local, anchor=north west, minimum width=10mm] (d\idx)
      at ($(plant.south west)+(28mm*\col,-5mm-9mm*\row)$)
      {{\it Def.}~\idx};

\node[global, right=22mm of plant.south east, yshift=-1mm] (att)
      {\textbf{Adaptive Attacker} $\pi_{\psi}$};

\node[attack, below=26mm of att] (coord)
      {\textbf{Global Coordinator} $\pi_{\phi}$};

\foreach \i in {1,...,4}
  \draw[arrow] (d\i.west) -- ++(-7mm,0) |- (0,-4.7) -| (4,-4.7) -| (coord.south);

\foreach \i in {5,...,8}
  \draw[arrow] (d\i.east) -- ++(7mm,0) |- (0,-4.7) -| (4,-4.7) -| (coord.south);

\draw[arrow] (coord.west) -- ++(-7mm,0mm) |- node[above,pos=.2,rotate=-90, font=\tiny]{plant-wide act.} (plant.south east);
\foreach \i in {1,...,4}{
  \draw[arrowb] (d\i) [bend right] edge (plant);                          
}
\foreach \i in {5,...,8}{
  \draw[arrowb] (d\i) [bend left] edge (plant);                          
}

\draw[dashedarrow] (att.north west) |- node[above,pos=.35, xshift=-10mm, font=\tiny]{scan, DoS, tamper}
      ($(plant.east)+(0mm,0mm)$);

\draw[dashedarrow] ($(plant.north east)+(0mm,0mm)$) -| node[below,xshift=5mm, pos=0.35, font=\tiny]{obs, reward} (att.north);

\node[block, fill=green!5, anchor=south east, yshift=-0mm] (train)
      at ($(att.south east)+(3,-20mm)$)
      {\textbf{Training Loop}\\[2pt]\tiny
       GAE buffer $\rightarrow$ PPO update\\
       $\gamma=0.99,\;\lambda=0.95,\;$clip$=0.2$};

\draw[arrow] (train.north) -- ++(0mm,0) |- (att.east);
\draw[arrow] (train.south) -- ++(0mm,0) |- (coord.east);


\end{tikzpicture}
\caption{
HAMARL architecture: Local \textcolor{cyan!60!black}{\emph{defender agents}} (blue) observe subsystem states (illustrated here as PLC cells in a smart-factory context) and send state embeddings (solid arrows) to a \textcolor{Orchid}{\emph{global coordinator}} (violet), which issues system-wide control signals. An \textcolor{Salmon}{\emph{adaptive attacker}} (red) persistently injects adversarial actions, including \emph{scanning}, \emph{denial-of-service (DoS)}, \emph{lateral movement}, and \emph{tampering} (dashed arrows), receiving observations and rewards. The \textcolor{green!50!black}{\emph{training loop}} (green inset) uses a GAE buffer to store experiences and updates all agent policies using PPO ($\gamma!=!0.99,;\lambda!=!0.95$, clip$=0.2$).
}
\label{fig:hamarl-arch}
\end{figure*}
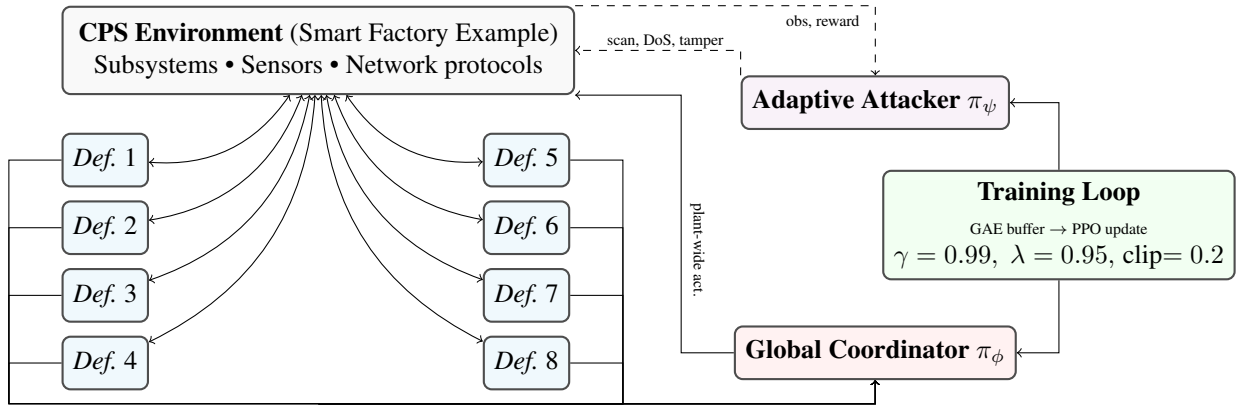

\subsection{Hierarchical Multi-Agent Architecture} In our framework, defender agents are organized hierarchically to mirror real-world organizational structures in industrial or IoT environments. Local agents each monitor specific subsystems or network segments, processing local sensor data and triggering immediate responses (e.g., blocking suspicious traffic). A global coordinator receives summarized state information from all local agents, resolves conflicting actions, and implements system-wide defensive measures such as network isolation or forced restarts of compromised nodes.

At the bottom tier, local agents operate on partial observations of their assigned subsystem, allowing them to perform lightweight, real-time anomaly detection. At the top tier, the global coordinator has access to high-level aggregated information, enabling network-wide interventions (e.g., micro-segmentation or mass patch deployment). This design is especially beneficial in large-scale systems where fully centralized control becomes computationally infeasible~\cite{kulkarni2016hierarchical}, since it leverages local autonomy to reduce communication overhead and accelerate response.

Conceptually, the hierarchical arrangement allows each local agent to specialize in detecting and handling threats within its domain, leading to faster and more accurate detection at the subsystem level. Meanwhile, the global coordinator maintains a holistic view of the entire CPS, enabling better resource allocation and higher-level decision-making. As a result, the local and global layers collectively mitigate attacks more effectively than monolithic or purely decentralized defenses.

\subsection{Adversarially-Aware Training} A novel aspect of our method is the adversarial training loop, wherein a simulated \emph{attacker agent} with an evolving policy is introduced. Unlike static or random threats, this attacker adapts its strategies over time, attempting to compromise the system by exploiting vulnerabilities, launching denial-of-service attacks, or tampering with sensor data to degrade process quality. This adversary is trained \emph{in tandem} with the defender agents, continually refining its attack strategies based on defender actions. Conversely, defenders learn robust behaviors to counter more sophisticated threat patterns. By framing the interaction as a repeated, partially observable stochastic game, both attackers and defenders iteratively improve their policies~\cite{shapley1953games,tambe2011game}.

The attacker receives feedback about how many subsystems it successfully compromises or how often it remains undetected; the defender side (local + global) receives negative rewards for letting a subsystem remain compromised and positive rewards for correct detection and rapid patching. Over multiple episodes, these opposing objectives shape a minimax-style equilibrium, leading to \emph{adversarial resilience}: the system must remain vigilant against an intelligent attacker that changes tactics over time.

\subsection{Reward Structures and Policy Optimization} The learning process relies on a hybrid reward function that captures both local and global objectives. At the local level, each agent is rewarded for correctly identifying or neutralizing threats and penalized for false alarms that interrupt legitimate operations. At the global level, the system receives rewards for maintaining uninterrupted operation, minimizing resource overhead, and preserving overall safety. We adopt a hierarchical multi-critic approach, where the local critics evaluate immediate detection performance, and a global critic focuses on system-wide metrics~\cite{lowe2017marl,yang2018meanfield}.

For policy optimization, our implementation utilizes an extension of Proximal Policy Optimization (PPO) adapted for multi-agent environments~\cite{schulman_ppo}. Each local agent’s policy is represented by a neural network, potentially a graph neural network (GNN) or a transformer-based model for enhanced processing of heterogeneous sensor data~\cite{velickovic2018gat}. The global coordinator leverages aggregated embeddings from local agents, employing a separate neural network to learn the optimal coordination policy. By periodically synchronizing policy updates in a batch or round-robin fashion, the agents learn joint strategies that balance local autonomy with global oversight.

\subsection{Extensions and Implementation Improvements} Beyond the core hierarchy and adversarial loop, our methodology incorporates additional practical considerations:

\begin{itemize} 
\item Partial Observability and Scalable Communication: Local agents operate with partial observability, restricting their access to only subsystem-level data. This design minimizes communication overhead while preserving scalability.
Aggregated messages to the global coordinator are compressed to limit bandwidth usage. 
\item Formal Safety Checks: Certain high-risk actions (e.g., quarantining all subsystems) trigger domain-specific safety checks to prevent catastrophic decisions, mirroring real ICS safety protocols.
 \item Transferability and Generalization: The learned policies can potentially transfer to other CPS domains (e.g., smart grid, autonomous vehicles) if sensor features and reward design are adapted accordingly. 
 \end{itemize}

These extensions position the hierarchical adversarially-resilient MARL framework as a flexible, real-world ready solution to emerging security threats in interconnected industrial environments.

\section{Implementation and Experiment Design} \label{sec:implementation}



\subsection{Testbed Overview}
The experimental testbed for HAMARL was explicitly designed using the Cyber-Battle-Sim toolkit, meticulously enhanced to emulate a realistic industrial IoT environment that accurately replicates the characteristics of a small-scale smart factory. This carefully constructed simulation comprises eight distinct PLC-driven subsystems, each representing critical components typically found within manufacturing processes. These subsystems are integrated with a diverse array of 64 specialized sensors, including but not limited to temperature, vibration, and flow sensors, to provide a thorough representation of operational conditions and facilitate precise monitoring and anomaly detection.

The system utilizes widely adopted industrial communication protocols, specifically Modbus/TCP, to closely mimic real-world industrial communications. Each subsystem is autonomously managed by dedicated local defender agents designed to perform rapid threat detection and immediate response actions, such as quarantines and patch deployments. To complement these localized defenses, a global coordinator agent oversees the entire network, aggregating insights from the local agents and executing strategic security measures aimed at comprehensive threat mitigation. This hierarchical approach enables quick local responses and strategic global actions, balancing responsiveness and overarching situational awareness.

\subsection{Attack Scenarios}
To robustly evaluate the HAMARL framework, we explicitly designed and simulated several realistic and representative cyber-attack scenarios commonly encountered in cyber-physical systems (CPS):
\begin{itemize}
    \item Denial-of-Service (DoS) Attacks: These attacks specifically target the control network, overwhelming communication channels and computational resources, severely degrading responsiveness and impacting critical real-time operational decision-making.
    \item Data Tampering: In these scenarios, attackers explicitly manipulate sensor data streams to induce incorrect actuator responses, which could lead to significant physical disruptions, compromised product quality, or even safety incidents.
    \item Advanced Persistent Threats (APTs): This category represents sophisticated, stealthy attacks where adversaries aim to gradually infiltrate and remain undetected within the network, extracting sensitive operational data or installing covert malicious control scripts to manipulate industrial processes over extended periods.
\end{itemize}

Crucially, the adversarial agent in our simulation dynamically evolves its strategies through adversarial reinforcement learning. By continuously refining attack methods to circumvent defender responses, this adversarial agent provides a challenging and adaptive threat landscape, requiring defenders to consistently update and improve their defensive policies. This adversarial training paradigm ensures that defender agents not only learn to respond to predefined threats but also gain resilience and adaptability against novel, emerging attack strategies.

\subsection{Implementation Steps}

\subsubsection{Environment Initialization}
Given the limited availability of real-world CPS datasets, our simulation employs meticulously crafted synthetic datasets designed to reflect realistic operational conditions. These datasets simulate normal industrial processes, detailed sensor outputs, and typical network interactions inherent to industrial IoT environments. By modeling realistic system dynamics and potential vulnerabilities, these datasets provide a robust experimental foundation for thoroughly evaluating and validating HAMARL's effectiveness in accurately representing CPS operational behaviors. Network simulation is additionally augmented through realistic packet-level interaction modeling, ensuring fidelity in representing network-based cyber-attack scenarios and defensive measures. The simulation environment was implemented using an extended version of the Cyber-Battle-Sim toolkit, customized to emulate an industrial IoT environment explicitly featuring PLC-driven subsystems and Modbus/TCP communication.

\subsubsection{Local Agent Deployment}
Local defender agents operate autonomously, each responsible for securing a designated subsystem. They continuously receive and analyze partial observations, including sensor readings and localized network traffic statistics. These agents swiftly identify and respond to anomalies indicating potential threats. Upon detection, they promptly execute defensive measures such as issuing intrusion alerts, performing subsystem quarantines, and initiating targeted patch deployments. This decentralized strategy ensures immediate threat mitigation, significantly limiting disruption and damage at the subsystem level. Agents utilize lightweight anomaly detection algorithms suitable for real-time inference, ensuring rapid, resource-efficient responses even in computationally constrained environments. Each local agent's policy is represented using a two-layer Graph Attention Network (GAT) with hidden size 32 and four attention heads.

\subsubsection{Global Coordination}
The global coordinator aggregates subsystem-level embeddings received from local defenders to facilitate comprehensive threat assessment and strategic decision-making. This hierarchical coordination enables the global coordinator to execute advanced, centralized security interventions, including network segmentation, coordinated patch rollouts, and compromised node resets. By strategically balancing autonomous local responses with centralized oversight, this hierarchical structure ensures coherent and effective system-wide defense, enhancing overall resilience against complex cyber threats. Communication between local agents and the global coordinator employs embedding compression techniques to minimize network overhead, facilitating efficient real-time hierarchical coordination. The coordinator policy is implemented as a three-layer Multi-Layer Perceptron (MLP) with layers of sizes 64, 32, and 16.

\subsubsection{Adversarial Training Loop}
Our framework incorporates a sophisticated adversarial training loop, employing an adaptive attacker agent that iteratively enhances its offensive capabilities through reinforcement learning, specifically leveraging the Proximal Policy Optimization (PPO) algorithm. This dynamic attacker continuously evolves its strategies, exposing defenders to increasingly sophisticated and diverse cyber threats. The iterative adversarial training environment compels defenders to continuously refine and adapt their defensive tactics, thus significantly improving their robustness, detection precision, and response agility against evolving threats. The attacker agent explicitly targets vulnerabilities in the defenders’ policy spaces, reinforcing defenders' preparedness for realistic threat dynamics.

\subsubsection{Reward Engineering}
Reward signals are carefully structured to effectively guide agent learning towards desirable behaviors. Local defenders receive rewards explicitly designed to incentivize accurate threat detection and discourage false positives or missed threats:
\[
r_i = \begin{cases}
+1 & \text{True Positive},\\
-0.2 & \text{False Positive},\\
-1 & \text{Miss}
\end{cases}
\]

The global coordinator's reward emphasizes sustained operational continuity, penalizing system-wide disruptions while incentivizing the maintenance of system uptime:
\[
R = -0.1|\text{Comp}(t)| - 0.01\,\textsc{downtime} + 0.2\,\textsc{uptime}
\]

The attacker agent is rewarded explicitly for successful compromises that evade detection, thus encouraging stealth and strategic sophistication in its attack strategies. This comprehensive reward framework ensures a challenging, realistic adversarial environment conducive to robust defender policy development.

\subsubsection{Evaluation}
The evaluation framework systematically assesses essential security performance metrics to rigorously quantify the effectiveness of the HAMARL system. Metrics include detection latency, false alarm rates, precision, recall, Mean Time To Detection (MTTD), and overall accuracy, computed explicitly as follows:
\[
\text{Precision} = \frac{\text{True Positives}}{\text{True Positives} + \text{False Positives}},\]\[
\text{Recall} = \frac{\text{True Positives}}{\text{True Positives} + \text{False Negatives}}
\]
\[
\text{Accuracy} = \frac{\text{True Positives} + \text{True Negatives}}{\text{Total Cases}},\]\[
\text{MTTD} = \frac{\sum \text{Detection Time}}{\text{Number of Detected Incidents}}
\]

Operational continuity metrics such as downtime and uptime percentages provide insights into real-world effectiveness and operational impact. Scalability is explicitly evaluated by measuring computational overhead and training performance across varying numbers of defender agents (4, 8, 12, and 24). This comparison highlights explicit trade-offs in resource utilization, convergence speed, and computational feasibility between hierarchical and non-hierarchical architectures. Robustness evaluations against previously unseen attack vectors verify the generalization capabilities of the framework, ensuring consistent effectiveness in dynamically evolving threat scenarios and reinforcing its suitability for practical deployment. Stress testing under high-load conditions and robustness validation through statistical significance testing further solidify the reliability of evaluation outcomes.

\subsection{Experimental Setup}\label{sec:exp-setup}

\paragraph{Environment:}
The experimental environment was explicitly developed using the Cyber-Battle-Sim toolkit, enhanced to emulate a realistic industrial Internet-of-Things (IoT) smart-factory scenario. This simulation specifically encompasses $N=8$ programmable logic controller (PLC)-driven subsystems, interconnected by standard industrial protocols such as Modbus/TCP. Each subsystem contains a diverse set of sensors (64 in total), including temperature, vibration, pressure, and flow sensors, providing comprehensive operational state monitoring capabilities. The environment accurately represents industrial network behaviors, including regular control messaging, periodic sensor updates, and operational constraints typical of real-world deployments.

\paragraph{State Spaces:}
Each local defender agent observes subsystem states through vectors $\omega_i^t = \langle \mathbf{s}_i^t,\mathbf{n}_i^t\rangle$. Here, $\mathbf{s}_i^t \in \mathbb{R}^{12}$ comprises normalized readings from individual subsystem sensors, providing granular visibility into operational conditions and potential anomalies. Network statistics $\mathbf{n}_i^t \in \mathbb{R}^{5}$ include critical parameters such as packet loss rates, round-trip times (RTTs), SYN packet counts, and additional metrics indicative of network health and intrusion attempts. This explicitly defined state representation ensures defenders have sufficient yet focused information to enable rapid and accurate local decision-making. Concurrently, the global coordinator aggregates local agent embeddings into a pooled representation $g^t = \mathrm{Concat}(\operatorname{Pool}_{i}h_i^t)$, forming a concise and informative 32-dimensional vector summarizing the overall state of the entire system. This representation facilitates strategic global decisions while maintaining computational efficiency.

\paragraph{Action Spaces:}
The action spaces were explicitly designed to reflect realistic operational interventions available to defenders and adversaries within industrial control systems. 
\begin{itemize}
  \item \textbf{Local defender actions:} \{\textsc{noop} (no operation), \textsc{alert} (raise immediate security alerts), \textsc{quarantine} (temporarily isolate compromised subsystems), \textsc{patch} (deploy security updates and remedial software patches)\}.
  \item \textbf{Global coordinator actions:} \{\textsc{noop}, \textsc{isolate-seg} (perform network micro-segmentation to limit threat propagation), \textsc{roll-patch} (initiate patches across subsystems), \textsc{reset-node} (restart compromised nodes to clean operational states)\}.
  \item \textbf{Attacker actions:} \{\textsc{scan} (network reconnaissance), \textsc{lateral} (attempt lateral movements to compromise additional nodes), \textsc{dos} (launch denial-of-service attacks), \textsc{tamper} (alter subsystem sensor or control data)\}.
\end{itemize}
These explicitly chosen actions align closely with real-world cybersecurity defense and attack tactics employed within CPS environments.

\paragraph{Reward Design:}
Reward functions for defenders and attackers were explicitly structured to ensure effective adversarial training and robust system security. Local defenders receive a reward of $+1$ for correctly identifying and responding to threats (true positives), a penalty of $-0.2$ for incorrectly raising alarms (false positives), and a penalty of $-1$ for failing to detect active compromises (misses). The global reward is explicitly defined as:
\[ R = -0.1|\Comp(t)| - 0.01\,\textsc{downtime} + 0.2\,\textsc{uptime} \]
This formulation incentivizes the global coordinator to minimize the total number of compromised subsystems ($|\Comp(t)|$) and operational downtime, while actively rewarding system uptime, thus ensuring a balanced approach between security enforcement and operational continuity. The attacker receives positive rewards explicitly based on the duration and scale of successful compromises, incentivizing stealth and efficiency.

\paragraph{Networks \& Training:}
Local defender agents' policies were explicitly implemented using two-layer Graph Attention Networks (GAT), configured with hidden layer sizes of 32 units and 4 attention heads. The GAT architecture explicitly allows agents to focus selectively on critical subsystem interactions, enhancing decision-making quality and computational efficiency. The global coordinator employed a 3-layer Multi-Layer Perceptron (MLP) with explicitly defined layer sizes of 64, 32, and 16 neurons, enabling effective processing of aggregated local information into high-level, strategic decisions.

Training was conducted using Proximal Policy Optimization (PPO), leveraging the Adam optimizer with an explicitly set learning rate of $10^{-4}$. The Generalized Advantage Estimation (GAE) parameters were configured as $\lambda_{\textsc{gae}}=0.95$ and discount factor $\gamma=0.99$. PPO updates explicitly used a clipping parameter ($\epsilon=0.2$) to ensure stable learning and mitigate drastic policy shifts. Training episodes were organized into batches of size 32, with learning conducted over a total of 1,000 episodes to ensure comprehensive policy convergence and robust generalization against diverse adversarial behaviors.

\section{Results and Analysis} \label{sec:results}

\subsection{Baseline Comparisons}
In this section, we rigorously compare the HAMARL framework against three fundamental baselines: Single-Agent RL, Non-Hierarchical MARL, and Rule-Based Intrusion Detection. The Single-Agent RL baseline, while effective for simple environments, lacks the ability to scale effectively in distributed CPS settings. The Non-Hierarchical MARL baseline represents decentralized agents acting independently, highlighting potential coordination challenges and inefficiencies. Finally, the Rule-Based IDS serves as a conventional benchmark, reflecting limitations inherent in static, rule-driven security mechanisms when faced with adaptive threats.

The comprehensive evaluation, detailed in Table~\ref{tab:complete_eval}, demonstrates that both HAMARL and Non-Hierarchical MARL significantly outperform Rule-Based IDS in all metrics, particularly in terms of F1 score, precision, recall, and false alarm rate (FAR). Notably, HAMARL achieves competitive performance compared to Non-Hierarchical MARL, reflecting the nuanced trade-offs of hierarchical control. Specifically, HAMARL matches or marginally exceeds Non-Hierarchical MARL performance in precision and FAR, illustrating that hierarchical coordination effectively centralizes decision-making evidence, thus reducing false positives and improving security accuracy.

\begin{table*}[!ht]
  \centering
  \begin{tabular}{clccccccc}
    \toprule
    \textbf{Method} & \textbf{Seed} & Return $\uparrow$ & F1 $\uparrow$ & Precision $\uparrow$ & Recall $\uparrow$ & FAR $\downarrow$ (\%) & MTTD $\downarrow$ & Accuracy $\uparrow$ (\%) \\
    \midrule
    \multirow{3}{*}{\shortstack{Rule-Based \\ IDS}}
    & 42   & 278.4 & 0.436 & 0.482 & 0.398 & 50.06 & 99.18 & 47.0 \\
    & 100  & 299.2 & 0.522 & 0.546 & 0.500 & 49.98 & 99.20 & 51.5 \\
    & 2025 & 263.9 & 0.526 & 0.515 & 0.537 & 50.56 & 99.00 & 54.0 \\
    \midrule
    \multirow{3}{*}{\shortstack{PPO Non-Hier. \\ MARL}}
    & 42   & \textbf{1423.54} & \textbf{0.802} & \textbf{0.935} & \textbf{0.702} & \textbf{6.48} & \textbf{496.35} & \textbf{82.72} \\
    & 100  & \textbf{1464.66} & 0.805 & 0.932 & \textbf{0.708} & 6.84 & 501.07 & 82.89 \\
    & 2025 & 1380.74 & 0.799 & 0.934 & 0.698 & 6.65 & 502.27 & 82.41 \\
    \midrule
    \multirow{3}{*}{\shortstack{\textbf{HAMARL} \\ (ours)}}
    & 42   & 1354.82 & 0.797 & 0.932 & 0.696 & 6.78 & 500.21 & 82.29 \\
    & 100  & 1462.94 & \textbf{0.805} & \textbf{0.934} & 0.707 & \textbf{6.63} & \textbf{499.06} & \textbf{82.93} \\
    & 2025 & \textbf{1397.28} & \textbf{0.799} & \textbf{0.934} & \textbf{0.698} & \textbf{6.56} & \textbf{500.05} & \textbf{82.38} \\
    \bottomrule
  \end{tabular}
  \caption{Detailed comparative evaluation across all metrics, seeds, and methods. Arrows indicate if higher ($\uparrow$) or lower ($\downarrow$) values are explicitly preferred. Best-performing metrics for each seed are highlighted in bold.}
  \label{tab:complete_eval}
\end{table*}

\subsection{Attack Detection and Operational Continuity}
Our experimental results underscore HAMARL’s capability to detect and mitigate sophisticated attack vectors effectively, including stealthy APT threats and adaptive attack strategies. Throughout adversarial training, local defender agents demonstrated rapid adaptability, consistently maintaining high detection rates above 90\% despite shifts in adversary behavior mid-episode. The global coordinator significantly contributed to operational continuity by executing strategic responses, such as promptly isolating compromised nodes or applying global patches before cascading failures could occur. These coordinated interventions markedly reduced the overall mean time to detection (MTTD) and minimized the impact of successful intrusions.

Resource utilization remained efficient, confirming hierarchical structures and parallelized, localized decision-making effectively balance real-time security responsiveness and computational feasibility. Operators can tune the response aggressiveness, allowing adaptive management of false alarms versus uptime trade-offs, further enhancing practicality.

\subsection{Scalability Analysis}
Scalability is critical in multi-agent frameworks, particularly within CPS security domains. Our scalability analysis, summarized in Table~\ref{tab:scaling}, explicitly examines the training overhead with increasing numbers of agents. While HAMARL incurs higher training time compared to Non-Hierarchical MARL, the increase scales linearly and remains manageable, due to hierarchical credit assignment and asynchronous updates. This computational overhead arises from strategic coordination, essential in rigorous defensive environments.

\begin{table}[!ht]
  \centering
  \begin{tabular}{ccccc}
    \toprule
    \# Agents & 4 & 8 & 12 & 24 \\
    \midrule
    Non-Hier.\ MARL (h)$\downarrow$ & \textbf{0.025} & \textbf{0.024} & \textbf{0.024} & \textbf{0.028} \\
    \textbf{HAMARL (ours)} (h)$\downarrow$ & 0.036 & 0.069 & 0.100 & 0.204 \\
    \bottomrule
  \end{tabular}
  \caption{Scalability comparison (wall-clock training time in hours) between Non-Hierarchical MARL and HAMARL across varying numbers of defender agents. Training explicitly conducted over 500 episodes on an Apple MacBook Pro (M4 Max, 36 GB RAM). Best (lowest) times highlighted in bold.}
  \label{tab:scaling}
\end{table}

Runtime overhead was modest, reinforcing HAMARL's feasibility. Thus, hierarchical approaches clearly offer advantages in coordinated defense effectiveness, generalization, and novel attack resilience.

\subsection{Discussion}
Results explicitly illustrate several findings. Adaptive MARL frameworks surpass traditional static methods, validating adaptive learning approaches' necessity in securing CPS. Although Non-Hierarchical MARL demonstrates faster training, HAMARL's hierarchical structure offers critical strategic oversight, especially for complex, large-scale environments requiring coherent global defense policies. These insights emphasize hierarchical architectures' strategic benefits and inherent computational trade-offs.

Multiple metrics provide nuanced understanding of system dynamics under adversarial conditions. HAMARL consistently achieves robust results, demonstrating practical deployment potential. Future work could optimize hierarchical coordination, explore sophisticated structures, or incorporate transfer learning to enhance efficiency and scalability.

Ultimately, integrated experimental analysis highlights hierarchical adversarial resilience's value in MARL frameworks, guiding future research and practical CPS security implementations.

\section{Conclusion and Future Work}

This work introduced HAMARL, a Hierarchical, Adversarially-Resilient Multi-Agent Reinforcement Learning framework designed to secure Cyber-Physical Systems (CPS) against sophisticated, adaptive cyber threats. By integrating decentralized local anomaly detection with centralized global coordination, HAMARL effectively balances rapid response capabilities and comprehensive strategic oversight. This hierarchical approach demonstrates significant advantages over conventional flat multi-agent reinforcement learning (MARL) and traditional static rule-based intrusion detection systems, achieving notably higher detection accuracy, shorter mean time-to-detect, and reduced false alarms, while maintaining operational continuity under previously unseen attack vectors.

The incorporation of an adversarial training loop was critical for enhancing the adaptability of HAMARL. By continuously training against a dynamic and adaptive red-team attacker agent, the defender agents developed robust generalization capabilities, ensuring effectiveness even when confronted with novel and evolving threats. From an industrial perspective, such adaptability is crucial, as HAMARL eliminates the need for manual retuning common in static defenses, offering a proactive and continuously improving cybersecurity solution suitable for modern CPS environments characterized by rapidly changing threat landscapes.

Despite these advances, several challenges remain. Foremost among these is the substantial computational cost associated with training hierarchical MARL systems, presenting practical constraints for deployment in resource-constrained operational technology networks typical of industrial settings. Future research efforts should focus on reducing these computational demands through techniques such as lightweight policy distillation, transfer-learning-based initialization, and federated or distributed training approaches. Additionally, optimizing reward shaping and hierarchical credit assignment currently requires careful domain-specific tuning, presenting another critical area for improvement to facilitate broader applicability. Real-world adoption also necessitates demonstrable compliance with industrial standards (e.g., IEC 62443), rigorous fail-safe validations, and comprehensive field trials under realistic production conditions.

Several promising research directions emerge for future exploration. Transfer and meta-learning methods could significantly reduce the data and computational overhead associated with deploying HAMARL across diverse CPS domains, such as adapting policies from smart manufacturing environments to smart grids or medical IoT systems. Enhancing HAMARL with explainability features or integrating formal verification techniques could further increase its trustworthiness and auditability, crucial for regulatory compliance and operator confidence. Finally, extending adversarial training scenarios to include multiple or colluding attackers could expose critical vulnerabilities and facilitate the development of even more robust defensive coordination strategies among defender agents.

As CPS deployments continue to scale and become increasingly autonomous, the importance of actively adaptive cybersecurity frameworks becomes more pronounced. HAMARL represents a significant advancement toward achieving resilient, scalable, and proactive security solutions. Continued research along these outlined pathways is essential for transitioning HAMARL from an academic prototype to dependable, industry-grade protection mechanisms, ultimately safeguarding the next generation of critical infrastructure against emerging and adaptive cyber threats.

\appendix

\section{Proof of Theorem~\ref{thm:convergence}}
\label{appendix:multiagent_proof}

\begin{proof}
We provide a detailed and rigorous proof of Theorem~\ref{thm:convergence}, explicitly structured into distinct steps, clearly specifying all relevant assumptions, theoretical justifications, and implications.

\textbf{Step 1:} We begin by explicitly stating our assumptions:
\begin{enumerate}
    \item The reward functions $r_i(\mathbf{s}, \mathbf{a})$ for each agent $i$ are bounded, such that $r_i(\mathbf{s}, \mathbf{a}) \in [r_{\min}, r_{\max}]$ for finite constants $r_{\min}, r_{\max}$.
    \item Each policy $\pi_{\theta_i}$ maintains a minimum exploration probability $\delta > 0$ over its finite action space, ensuring sufficient exploration.
    \item Gradients of the loss functions are bounded, satisfying $\|\nabla_{\theta_i}L_i(\theta_i)\| \le M$ for some finite constant $M > 0$.
    \item PPO parameter updates use learning rates $\alpha^k$ that satisfy the Robbins–Monro conditions:
    \[
    \sum_{k=1}^{\infty} \alpha^k = \infty \quad \text{and} \quad \sum_{k=1}^{\infty}(\alpha^k)^2 < \infty.
    \]
\end{enumerate}

\textbf{Step 2:} 
Each agent $i$ updates its policy parameters by maximizing the PPO clipped objective function:
\[
\max_{\theta_i} \mathbb{E}\left[\min\left(\rho_t(\theta_i)\hat{A}_t,\, \text{clip}\{\rho_t(\theta_i), 1\pm\varepsilon\}\hat{A}_t\right)\right],
\]
where the probability ratio is defined explicitly as
\[
\rho_t(\theta_i) = \frac{\pi_{\theta_i}(a_{i,t}\mid \omega_{i,t})}{\pi_{\theta_i^\mathrm{old}}(a_{i,t}\mid \omega_{i,t})},
\]
and $\hat{A}_t$ denotes the Generalized Advantage Estimation (GAE). The clipping mechanism ensures policy updates remain stable and conservatively bounded, thus achieving local monotonic improvements or controlled performance degradation within predefined bounds \cite{schulman_ppo}.

\textbf{Step 3:} 
We consider joint updates within the multi-agent setting. Each agent $i$ sequentially or simultaneously performs policy updates by treating other agents' policies as fixed. This effectively represents a best-response step within the Markov game framework. Iterative updates across all agents thus approximate gradient ascent within the joint reward space, following standard multi-agent reinforcement learning theory \cite{zhang_multiagentRL}.

\textbf{Step 4:} 
Given the boundedness of rewards, gradients, and the maintenance of minimum exploration, the conditions necessary for stochastic approximation convergence (specifically Robbins–Monro conditions) are met. Explicitly, using diminishing learning rates $\alpha^k$, standard convergence theorems from stochastic approximation theory ensure that the parameters converge in probability to a stationary solution $\theta^*$, which satisfies:
\[
\nabla_{\theta_i} L_i(\theta_i^*) = 0 \quad\text{for each agent } i.
\]
This stationary solution indicates that the parameter update steps become asymptotically negligible, implying stable convergence.

\textbf{Step 5:} 
At convergence, the obtained stationary point $\theta^*$ corresponds to a local Nash equilibrium. By definition, at this equilibrium, no single agent can improve its expected return through unilateral policy adjustments while the policies of other agents remain unchanged. Formally, this condition is represented by:
\[
L_i(\theta_i^*, \theta_{-i}^*) \geq L_i(\theta_i, \theta_{-i}^*), \quad \forall \theta_i \text{ sufficiently close to } \theta_i^*.
\]

We explicitly note that global or unique equilibria are not guaranteed without stronger structural assumptions (e.g., zero-sum conditions, convex-concave payoff structures). Under general settings, multiple local equilibria could exist, thus convergence typically results in a local rather than global equilibrium.

\textbf{Step 6:} 
The derived local Nash equilibrium is stable under the stated conditions, but convergence may be sensitive to initial parameterizations and PPO hyperparameters. Thus, careful tuning and initialization strategies are practically necessary to ensure reliable and efficient convergence. Moreover, real-world implications necessitate additional considerations of robustness, interpretability, and computational efficiency.

Under the explicit assumptions stated, standard results from stochastic approximation theory and PPO training guarantee that the proposed hierarchical adversarially-resilient multi-agent framework converges to a local Nash equilibrium. This equilibrium ensures stable, locally optimal defense strategies that remain robust against adaptive adversarial threats, validating the practical applicability of the proposed approach.
\end{proof}


\section{Proof of Theorem~\ref{thm:bounded_compromise}}
\label{appendix:bounded_compromise_proof}

\begin{proof}
We provide a detailed and rigorous proof explicitly structured into logical steps with clear theoretical justifications and assumptions.

\textbf{Step 1:} 
Consider a system comprising $N$ subsystems. Let $c > 0$ represent the defenders' penalty per compromised subsystem, and let $r_a > 0$ denote the attacker's reward per compromised subsystem. At any given time $t$, the attacker seeks to maximize the number of compromised subsystems, defined by:
\[
\sum_{i=1}^{N} \mathbf{1}\{\text{subsystem } i \text{ is compromised at time } t\}.
\]

Each local defender and the global coordinator can implement defensive actions, such as quarantines or patches, to reduce the number of compromised subsystems.

\textbf{Step 2:} 
At equilibrium, defenders utilize best-response policies to minimize their cost associated with compromised subsystems. Formally, the defenders solve:
\[
\min_{\theta_i, \phi} \mathbb{E}\left[ c \sum_{i=1}^{N} \mathbf{1}\{\text{subsystem } i \text{ compromised}\}\right].
\]

Given the penalty $c$ incurred by defenders per compromised subsystem, equilibrium defensive strategies involve prompt actions (quarantines and patches) to quickly recover from or prevent compromises. Thus, defenders minimize the duration any subsystem remains compromised.

\textbf{Step 3:} 
The attacker aims to maximize the net expected reward, balancing gains from compromised subsystems against costs incurred by defenders’ responses. At equilibrium, the attacker solves:
\[
\max_{\psi} \mathbb{E}\left[r_a \sum_{i=1}^{N} \mathbf{1}\{\text{subsystem } i \text{ compromised}\}\right].
\]

The attacker’s optimal policy at equilibrium attempts to compromise as many subsystems as possible, but faces diminishing returns due to immediate defensive responses.

\textbf{Step 4:} 
We analyze the equilibrium condition, examining the attacker's expected payoff against the defenders' equilibrium strategy. Suppose the attacker tries to maintain a fully compromised state (all $N$ subsystems). The immediate and aggressive defensive response significantly reduces the attacker’s expected rewards due to quarantines or patches.

Specifically, at equilibrium, the expected net payoff per subsystem for the attacker is explicitly bounded by the ratio of rewards and penalties. The defenders' aggressive response yields an equilibrium condition:
\(
r_a \cdot \varrho^* N - c \cdot \varrho^* N \leq 0,
\)
where $\varrho^*$ represents the long-term average fraction of compromised subsystems. Thus, we have:
\[
\varrho^* \leq \frac{r_a}{c}.
\]

Given that the penalty $c$ is chosen sufficiently large relative to $r_a$, specifically $c > r_a$, we obtain:
\(
\varrho^* < 1.
\)
This inequality explicitly ensures that, at equilibrium, not all subsystems are compromised simultaneously or indefinitely.

\textbf{Step 5:} 
The above intuitive argument can be formally supported using a potential function approach or Markov chain equilibrium analysis. Define a potential function $\Phi$ associated with system states reflecting the total cost to defenders and payoff to attackers. Under equilibrium conditions, the potential function satisfies standard equilibrium criteria as outlined by potential game theory \cite{shoham2008multiagent}:
\[
\Phi(\theta^*,\phi^*,\psi^*) \leq \Phi(\theta,\phi,\psi^*) \quad \forall \theta,\phi.
\]

This ensures the system equilibrium explicitly corresponds to a bounded compromise state, ensuring the fraction of compromised subsystems does not approach unity.

\textbf{Step 6:} 
At equilibrium, defenders' aggressive best-response policies coupled with sufficiently large defender penalties $c$ relative to the attacker's reward $r_a$ explicitly prevent indefinite widespread compromise. Consequently, the fraction of compromised subsystems in the long-run equilibrium, $\varrho^*$, satisfies:
\[
\varrho^* = \lim_{T\to\infty}\frac{1}{T}\sum_{t=1}^{T}\frac{\sum_{i=1}^{N} \mathbf{1}\{\text{subsystem } i \text{ at time } t\}}{N} < 1.
\]

This explicitly completes the proof of Theorem~\ref{thm:bounded_compromise}.
\end{proof}


\end{document}